\documentclass{article}
\usepackage{iclr2027_conference,times}
\iclrfinalcopy
\makeatletter
\renewcommand{\@maketitle}{\vbox{\hsize\textwidth
  {\LARGE\scshape\@title\par}
  \vskip 12pt
  {\normalsize\raggedright\@author\par}
  \vskip 0.3in minus 0.1in}}
\makeatother
\renewcommand{\headrulewidth}{0pt}
\usepackage[utf8]{inputenc}
\usepackage[T1]{fontenc}
\usepackage[hidelinks]{hyperref}
\usepackage{url}
\usepackage{booktabs}
\usepackage{amsmath,amssymb,amsthm}
\usepackage[nointegrals]{wasysym}
\usepackage{graphicx}
\usepackage{microtype}
\usepackage{xcolor}
\usepackage{listings}
\usepackage{placeins}
\usepackage{colortbl}
\usepackage{tabularx}
\newcolumntype{C}{>{\centering\arraybackslash}X}
\usepackage{tcolorbox}
\usepackage{algpseudocode}
\usepackage{enumitem}

\definecolor{hackred}{HTML}{B03A2E}
\definecolor{idrfblue}{HTML}{1F5F8B}
\definecolor{idrftint}{HTML}{EAF1F7}
\newcommand{\idrfrow}{\rowcolor{idrftint}}
\newenvironment{propbox}{\begin{tcolorbox}[colback=idrftint, colframe=idrfblue!85,
  boxrule=0.6pt, arc=2pt, left=10pt, right=10pt, top=5pt, bottom=5pt,
  before skip=7pt, after skip=7pt]}{\end{tcolorbox}}
\newenvironment{algobox}[1]{\begin{tcolorbox}[colback=white, colframe=idrfblue!85,
  colbacktitle=idrfblue, coltitle=white, boxrule=0.6pt, arc=2pt,
  left=6pt, right=6pt, top=3pt, bottom=3pt, toptitle=2pt, bottomtitle=2pt,
  fonttitle=\small, title={#1}]\small}{\end{tcolorbox}}
\algrenewcommand{\algorithmiccomment}[1]{\hfill{\color{idrfblue}\footnotesize$\triangleright$\ #1}}
\algrenewcommand\algorithmicrequire{\textbf{Input:}}
\algrenewcommand\algorithmicensure{\textbf{Output:}}
\newcommand{\algsection}[1]{\Statex\hspace{-\algorithmicindent}{\color{idrfblue}\textit{#1}}}
\newcommand{\samplesep}[1]{\par\vspace{1.5pt}{\color{#1!35}\hrule height 0.4pt}\vspace{1.5pt}}
\tcbset{samplebox/.style={colback=white, boxrule=0.6pt, arc=1.2pt, left=3pt, right=3pt,
  top=1.5pt, bottom=1.5pt, toptitle=0.6pt, bottomtitle=0.6pt,
  fonttitle=\scriptsize\bfseries, fontupper=\scriptsize\linespread{0.95}\selectfont,
  before skip=0pt, after skip=0pt}}
\graphicspath{{figures/}{images/}}

\newcommand{\E}{\mathbb{E}}
\newcommand{\pth}{p_\theta}
\newcommand{\pref}{p_{\mathrm{ref}}}
\newcommand{\fref}{f_{\mathrm{ref}}}
\newcommand{\fpsi}{f_\psi}

\newcommand{\mask}{\mathbf m}
\newcommand{\Lmdlm}{\mathcal L_{\mathrm{NELBO}}}
\newcommand{\LID}{\mathcal L_{\mathrm{ID}}}
\newcommand{\LR}{\mathcal L_{\mathrm{R}}}

\newcommand{\KL}[2]{D_{\mathrm{KL}}\!\left(#1\,\|\,#2\right)}
\newcommand{\wt}{w_t}
\newtheorem{proposition}{Proposition}
\newtheorem{corollary}{Corollary}

\definecolor{idrfcomment}{RGB}{39,128,121}
\definecolor{idrfkeyword}{RGB}{64,72,190}
\definecolor{idrffunc}{RGB}{186,55,118}
\definecolor{idrfstring}{RGB}{170,58,54}

\lstdefinestyle{idrf}{
  basicstyle=\scriptsize\ttfamily,
  commentstyle=\color{idrfcomment},
  stringstyle=\color{idrfstring},
  keywordstyle=[1]\color{idrfkeyword}\bfseries,
  keywordstyle=[2]\color{idrffunc},
  morekeywords=[1]{if, else, for, in, with},
  morekeywords=[2]{backsim,trajectorysim,probs,sample_state,update,
                   L_seq,ID_loss,reward_loss,rloo_advantage,no_grad},
  comment=[l]{\#},
  morestring=[b]",
  mathescape=true,
  columns=fullflexible,
  keepspaces=true,
  showstringspaces=false,
  xleftmargin=0.4em,
  aboveskip=0.25em,
  belowskip=0.15em,
}
\newcounter{idrfalgorithm}

\title{IDRF: Inverse-Distilled Reward Fine-tuning\\of Masked Discrete Diffusion Models}
\author{%
\textbf{Vladislav Gromadskii}\textsuperscript{1}\qquad
\textbf{David Li}\textsuperscript{2}\qquad
\textbf{Samson Gourevitch}\textsuperscript{3}\\[4pt]
\textbf{Yazid Janati}\textsuperscript{2,4}\quad
\textbf{Eric Moulines}\textsuperscript{2,5}\quad
\textbf{Maxim Panov}\textsuperscript{2}\quad
\textbf{Alexander Korotin}\textsuperscript{1,6}\\[8pt]
{\small
\textsuperscript{1}Applied AI Institute, Moscow, Russia\quad
\textsuperscript{2}MBZUAI\\[2pt]
\textsuperscript{3}CMAP, \'{E}cole Polytechnique\quad
\textsuperscript{4}Institute of Foundation Models\\[2pt]
\textsuperscript{5}LRE, EPITA\quad
\textsuperscript{6}AXXX, Moscow, Russia\\[4pt]
\href{mailto:wladgrm@gmail.com}{\texttt{wladgrm@gmail.com}},
\href{mailto:iamalexkorotin@gmail.com}{\texttt{iamalexkorotin@gmail.com}}}%
}
\hypersetup{%
  pdftitle={IDRF: Inverse-Distilled Reward Fine-tuning of Masked Discrete Diffusion Models},
  pdfauthor={Vladislav Gromadskii, David Li, Samson Gourevitch, Yazid Janati, Eric Moulines, Maxim Panov, Alexander Korotin}%
}

\begin{document}
\maketitle

\begin{abstract}
Masked discrete diffusion models offer a promising alternative to autoregressive generation, but iterative sampling can be costly, and intractable sequence likelihoods complicate reward fine-tuning. We introduce \textbf{IDRF}, a framework for reward fine-tuning of few-step masked discrete diffusion generators. Starting from a standard reverse-KL-regularized objective, IDRF replaces the intractable sequence-level KL penalty with inverse-distillation regularization. With an optimal auxiliary denoiser, we prove that the population inverse-distillation loss upper-bounds the sequence-level KL divergence to the reference distribution. IDRF optimizes a trajectory-based surrogate of this loss without reference-model rollouts, so the student keeps its own few-step sampler. We view few-step generation as a finite-horizon Markov decision process and optimize reward with a clipped policy-gradient objective over the student's trajectories. Across DNA, image, and text generation, IDRF achieves high reward with up to \(32\times\) fewer denoising steps than the reference while mitigating reward hacking and preserving sample quality.
\end{abstract}

\begin{figure}[!ht]
  \centering
  \includegraphics[width=\linewidth]{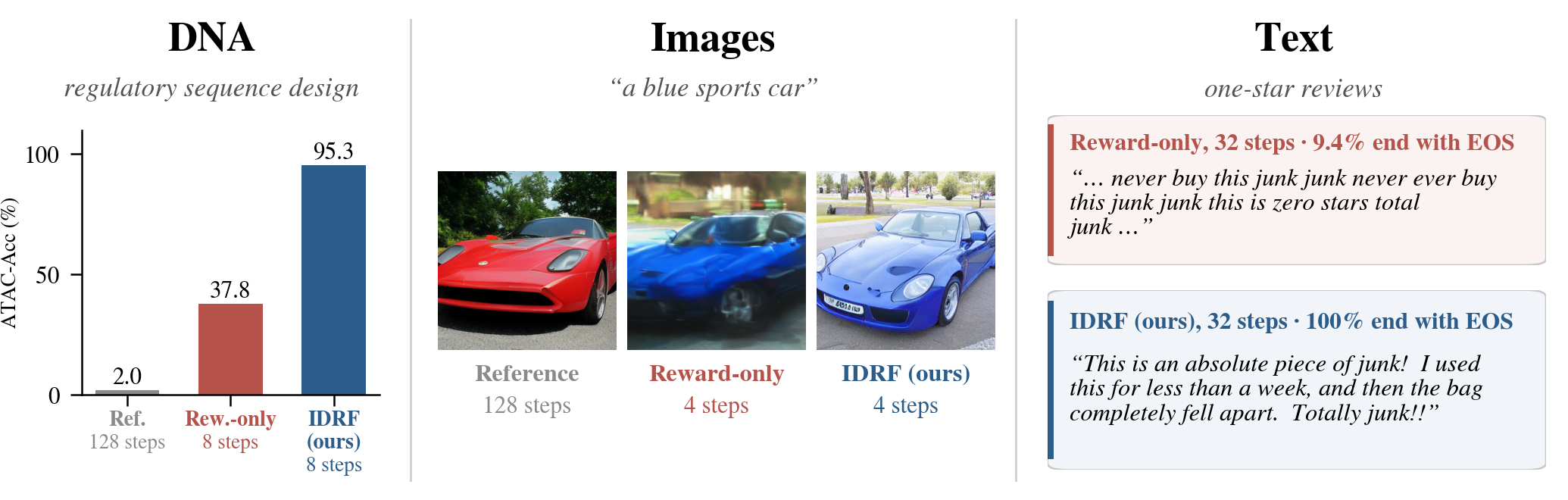}
  \caption{\textbf{IDRF (ours) across DNA, images, and text.} Reward-only fine-tuning and IDRF sample with the same budget of $8$, $4$, and $32$ steps. Reward-only fine-tuning hacks the reward: DNA lacks accessibility (ATAC-Acc), images blur, and reviews degenerate into repetition that rarely ends with EOS.
IDRF avoids these failures.}
  \label{fig:teaser}
\end{figure}

\section{Introduction}
Masked diffusion models (MDMs) generate sequences by unmasking tokens in parallel, with applications
to language, biological sequences, and discrete image representations
\citep{sahoo2024mdlm,shi2024md4,wang2024drakes}. Reward fine-tuning adapts these models to task
objectives, such as DNA activity, image attributes, or review sentiment. Optimizing reward alone,
however, can destroy properties the reward does not measure (Figure~\ref{fig:teaser}).

KL regularization toward a pretrained reference provides a standard safeguard
\citep{ziegler2019finetuning,korbak2022rl}, but an MDM's sequence likelihood marginalizes over all
denoising trajectories and is intractable. Existing reward fine-tuning methods use likelihood
surrogates and evidence bounds \citep{zhao2025d1,wang2025spg}, or penalize KL between reverse
transitions \citep{wang2024drakes}. The latter requires reference evaluations along the rollout
and ties the regularizer to its sampling path. We seek a tractable sequence-level KL regularizer
that lets the student use its own sampler.

We introduce \textbf{Inverse-Distilled Reward Fine-tuning (IDRF)}, which combines trajectory-based
policy gradients with inverse distillation \citep{li2026idlm}. The regularizer compares the denoising
losses of a frozen reference and a fitted auxiliary on student samples. It requires neither the
reference's likelihood nor its rollouts. Both the reward and the regularizer reuse the student's own
trajectories (Figure~\ref{fig:idrf-overview}), so the regularizer shapes the student's few-step
sampler directly. Reward fine-tuning requires neither pretraining data nor reward gradients. Our \textbf{contributions}
are as follows.
\begin{enumerate}[leftmargin=1.4em,itemsep=1pt,topsep=2pt]
\item With an optimal auxiliary denoiser, we establish \textbf{a sequence-level KL bound} for the population inverse-distillation objective in masked diffusion, valid for any sequence length and any student sampler (\wasyparagraph\,\ref{sec:method-id}).
\item We regularize reward fine-tuning with \textbf{a practical inverse-distillation objective} on the student's own trajectories,
without reference rollouts. The resulting framework supports reward fine-tuning directly in the few-step regime
(\wasyparagraph\,\ref{sec:student}).
\item Across DNA, images, and text, IDRF reaches \textbf{high reward with up to $\boldsymbol{32\times}$ fewer denoising steps} than the reference while preserving properties lost under reward-only fine-tuning
(\wasyparagraph\,\ref{sec:experiments}).
\end{enumerate}

\section{Background}\label{sec:background}
This section reviews the two ingredients of IDRF. \wasyparagraph\,\ref{sec:bg-mdm} introduces masked
diffusion and the NELBO used to pretrain the reference, and \wasyparagraph\,\ref{sec:bg-id} describes
inverse distillation, which IDRF turns into a regularizer for reward fine-tuning.

\subsection{Masked Diffusion}\label{sec:bg-mdm}
Masked diffusion trains a denoiser to recover hidden tokens from their visible context and
generates by filling in a fully masked sequence. We follow the MDLM formulation
\citep{sahoo2024mdlm,shi2024md4}, which ties this prediction task to a variational bound.

Let $x_0\in\mathcal V^L$ be a clean sequence and $\mask\notin\mathcal V$ the mask
symbol. At noise level $t\in[0,1]$, each token is retained with probability
$\alpha_t$ and masked otherwise, independently across positions. The resulting
corruption distribution is
\begin{equation}\label{eq:mdlm-forward}
q_t(x_t\mid x_0)
=\prod_{\ell=1}^L\left[\alpha_t\,\delta_{x_0^\ell}(x_t^\ell)
+(1-\alpha_t)\,\delta_{\mask}(x_t^\ell)\right],
\end{equation}
where $\delta_a$ is a point mass at $a$. The schedule $\alpha_t$ is continuously
differentiable and strictly decreasing, with $\alpha_0=1$ and $\alpha_1=0$.

The denoiser $f^\ell(v\mid x_t)$ predicts the missing token at position $\ell$
from the visible context of $x_t$ and takes no time input. Positions are masked independently
of their values, so the Bayes-optimal prediction depends only on the visible tokens of
$x_t$ and is independent of $t$ \citep{ou2025radd,scope_zheng2025timeagnostic}. Accordingly, time
conditioning has minimal effect in practice \citep{sahoo2024mdlm} and implementations omit
it \citep{ou2025radd}, as do the pretrained models in all our experiments. MDLM trains the
denoiser with the continuous-time negative evidence lower bound (NELBO)
\begin{equation}\label{eq:mdlm}
\Lmdlm(f,p)
=\E_{\substack{x_0\sim p,\,t\sim U(0,1)\\x_t\sim q_t(\cdot\mid x_0)}}
\left[\wt\sum_{\ell:x_t^\ell=\mask}
-\log f^\ell(x_0^\ell\mid x_t)\right],
\qquad
\wt=\frac{-\alpha'_t}{1-\alpha_t}.
\end{equation}
Here $p$ is the distribution supplying clean sequences, the data distribution $p^*$ during
pretraining, so learning to predict missing tokens optimizes a bound on sequence likelihood.

To generate a sample, the model starts from $x_1=\mask^L$. At each reverse step,
some masked positions are selected and their values are drawn from the
denoiser's predictions. MDLM's substitution-based (SUBS) parameterization
keeps revealed tokens fixed and assigns zero probability to predicting
$\mask$ as a clean token. For $0\le s<t\le1$, this gives the reverse transition
\begin{equation}\label{eq:mdlm-reverse}
\begin{aligned}
p_f(x_s\mid x_t)&=\prod_{\ell=1}^L p_f^\ell(x_s^\ell\mid x_t),\\
p_f^\ell(\cdot\mid x_t)&=
\begin{cases}
\delta_{x_t^\ell}, & x_t^\ell\ne\mask,\\[2pt]
\displaystyle\frac{1-\alpha_s}{1-\alpha_t}\,\delta_{\mask}
+\frac{\alpha_s-\alpha_t}{1-\alpha_t}\,f^\ell(\cdot\mid x_t),
& x_t^\ell=\mask.
\end{cases}
\end{aligned}
\end{equation}
Positions are updated independently given $x_t$, but newly revealed tokens condition later
predictions, so the generated tokens depend on one another.

\subsection{Inverse Distillation}\label{sec:bg-id}
IDLM uses inverse distillation to train a fast discrete generator from a
pretrained denoiser \citep{li2026idlm}. Pretraining fits a denoiser to a given
data distribution. Inverse distillation reverses this direction by fixing the
reference denoiser $\fref$ and learning a student distribution $\pth$ on which
its predictions remain optimal. In the ideal setting, a student matching
$p^*$ would produce samples that the reference can denoise as well as any
model fitted specifically to those samples.

A low reference loss alone is not enough to ensure this match. For example,
a student that repeats a single sequence may be easy to denoise despite
covering very little of the data distribution. Inverse distillation accounts
for how predictable the student's samples are by comparing the reference
with the best denoiser for those same samples. This gives the objective
\begin{equation}\label{eq:id}
\LID(\theta)
=\Lmdlm(\fref,\pth)-\min_f\Lmdlm(f,\pth).
\end{equation}
Minimizing this gap encourages the student to produce samples on which
refitting the denoiser offers little improvement over the reference.

This objective also has a distribution-matching guarantee. With a Bayes-optimal reference for
$p^*$ and an exact inner minimum over a sufficiently expressive denoiser family, IDLM shows
\citep[Theorem~3.1]{li2026idlm}
\begin{equation}\label{eq:idlm-marginal-bound}
\LID(\theta)\geq\KL{\pth}{p^*}\geq0,
\end{equation}
so a zero gap means that the student matches the data. IDLM proves this bound for a single token and
for a joint reverse process over whole sequences, and leaves the sequence objective used in practice,
a sum of per-token denoising losses, to future work \citep[App.~D.2]{li2026idlm}.
Proposition~\ref{prop:sequence-kl} proves the bound for its population loss and an arbitrary reference.

In practice, a trainable auxiliary denoiser $\fpsi$ approximates the inner
minimum by learning to reconstruct masked student samples. Training alternates
between fitting the auxiliary and updating the student to reduce the difference
between reference and auxiliary losses. The reference remains fixed throughout. IDRF uses this denoising gap as the regularizer for
reward fine-tuning.
\section{Method}\label{sec:method}
We start from the standard KL-regularized reward objective and use inverse distillation as its
tractable regularizer. Figure~\ref{fig:idrf-overview} illustrates the resulting training framework.

\begin{figure}[t]
  \centering
  \includegraphics[width=\linewidth]{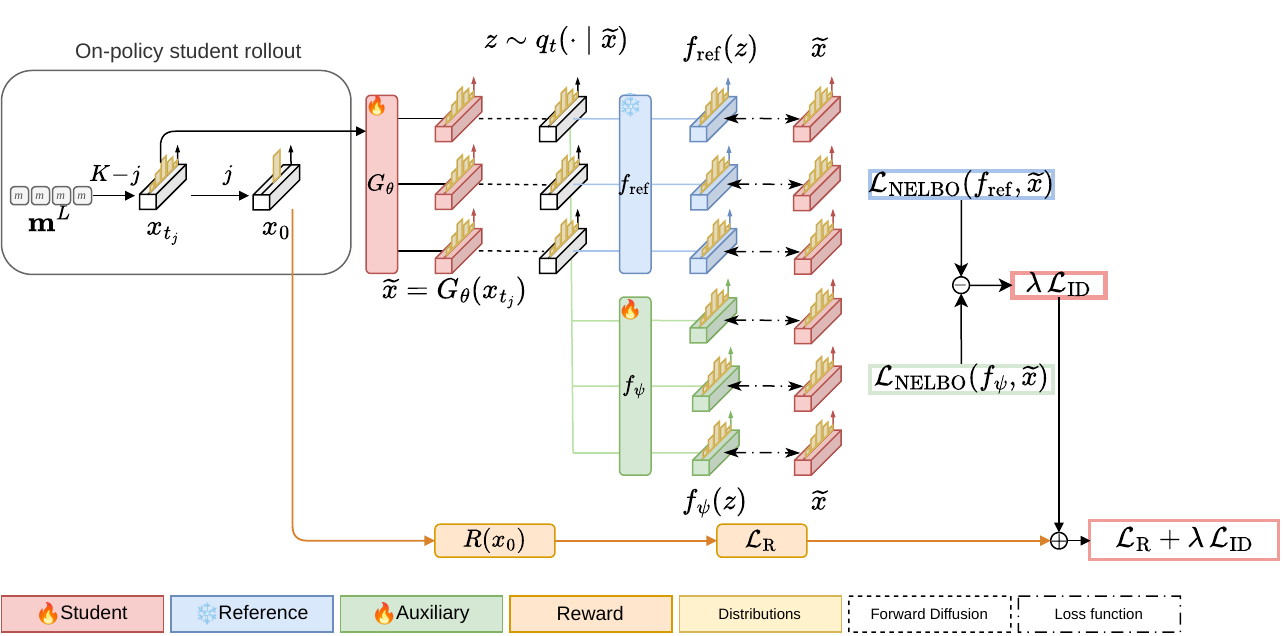}
  \caption{\textbf{IDRF (ours) training framework.} An on-policy rollout of the student from $\mathbf m^L$
  provides an intermediate state $x_{t_j}$ and a completed sample $x_0$. The soft prediction
  $\widetilde x=G_\theta(x_{t_j})$ is freshly corrupted and scored by the frozen reference $\fref$ and
  the trainable auxiliary $\fpsi$, whose NELBO gap gives $\LID$. The reward of $x_0$ gives the
  policy-gradient loss $\LR$, and the student minimizes $\LR+\lambda\LID$.}
  \label{fig:idrf-overview}
\end{figure}

\subsection{KL-regularized reward fine-tuning via inverse distillation}\label{sec:method-id}
Given a terminal reward $R$ and a frozen reference distribution $\pref$, the target objective is
\begin{equation}\label{eq:objective}
\max_\theta\,
\E_{x_0\sim\pth}[R(x_0)]
-
\lambda\KL{\pth}{\pref},
\qquad \lambda>0.
\end{equation}
The endpoint law marginalizes over all denoising trajectories ending at $x_0$
(\wasyparagraph\,\ref{sec:student}), so its log-density is generally intractable. We instead use the inverse-distillation gap in Eq.~\eqref{eq:id}.
The following proposition makes its connection to the KL between full sequences explicit.

\begin{propbox}
\begin{proposition}[Sequence-level KL bound]\label{prop:sequence-kl}
Let $\pth$ be any law on $\mathcal V^L$. Assume independent forward masking with a
continuously differentiable, strictly decreasing schedule with $\alpha_0=1$ and $\alpha_1=0$.
Let $\pref$ be the endpoint law of the continuous-time SUBS reverse process defined by
$\fref$, with unmasking rates $w_t\fref^\ell(v\mid x_t)$.
Assume the reference NELBO is finite and the auxiliary is optimal, that is, its family contains
the Bayes denoiser of $\pth$ on forward-masked states and attains the inner minimum. Then
\begin{equation}\label{eq:sequence-kl-bound}
\LID(\theta)
:=\Lmdlm(\fref,\pth)-\min_\psi\Lmdlm(\fpsi,\pth)
\ \geq\ \KL{\pth}{\pref}.
\end{equation}
Equality holds if the reference is the Bayes denoiser of its own endpoint law.
\end{proposition}
\end{propbox}

\paragraph{Why the population loss.}
Proposition~\ref{prop:sequence-kl} holds for any sequence length and any student sampler
(Appendix~\ref{app:sequence-kl}). Replacing the KL in Eq.~\eqref{eq:objective} by the population loss
therefore gives a lower bound on the KL-regularized objective and motivates the regularizer of IDRF.

\paragraph{Practical objective.}
IDRF optimizes a practical objective built from the population loss. A trained auxiliary takes the
place of the Bayes denoiser of $\pth$, and the loss is evaluated on the student's predictions
$\widetilde x=G_\theta(x_{t_j})$ at its rollout states $x_{t_j}$, $j\sim\mathcal U\{1,\dots,K\}$:
\begin{equation}\label{eq:idrf}
\min_\theta\max_\psi\;
\LR(\theta)+\lambda\,\E_{j,\,x_{t_j}}\big[
\Lmdlm(\fref,\widetilde x)-\Lmdlm(\fpsi,\widetilde x)\big],
\end{equation}
with $\LR$ and the soft-target NELBO defined in \wasyparagraph\,\ref{sec:reward-term}. In the experiments, $\LID$ and
$\lambda_{\mathrm{ID}}$ denote this term and its weight. With an optimal auxiliary, the expected term
upper-bounds $\KL{\bar q_\theta}{\pref}$ for the mixture $\bar q_\theta$ of the student's one-step completion
laws along its rollouts (Appendix~\ref{app:practical-bound}).

\subsection{The Few-step Generator as an MDP}\label{sec:student}\label{sec:reward-term}
We initialize the student denoiser $G_\theta$ from the pretrained reference. Like the reference,
it maps a partially masked sequence to categorical probabilities over clean tokens at each masked
position and takes no time input. Task conditioning is suppressed in the notation. A $K$-step
generator applies $G_\theta$ on a fixed grid $1=t_K>\cdots>t_0=0$. At step $k$, each masked
position is revealed with probability $b_k=(\alpha_{t_{k-1}}-\alpha_{t_k})/(1-\alpha_{t_k})$ and its
value is drawn from $G_\theta$, while revealed positions are copied. This is the SUBS transition of Eq.~\eqref{eq:mdlm-reverse} on the grid, and
$b_1=1$ ends every episode at a complete sequence $x_0$ after $K$ network evaluations.
Max-confidence decoding instead reveals the most confident masked positions and leaves the rest of
the construction unchanged. We formalize
the generator as the finite-horizon Markov decision process
\begin{equation}\label{eq:mdp}
\left\{\;
\begin{aligned}
&\text{state} && s_k=(k,\,x_{t_k}),\qquad s_K=(K,\,\mask^L),\\
&\text{action} && a_k=x_{t_{k-1}}\in\mathcal A(x_{t_k})
 =\bigl\{y:\ y^\ell=x_{t_k}^\ell\ \text{whenever}\ x_{t_k}^\ell\neq\mask\bigr\},\\
&\text{policy} && \pi_\theta(a_k\mid s_k)
 =\prod_{\ell:\,x_{t_k}^\ell=\mask}\Bigl[(1-b_k)\,\delta_{\mask}(a_k^\ell)
 +b_k\,G_\theta^\ell(a_k^\ell\mid x_{t_k})\Bigr],\\
&\text{transition} && s_{k-1}=(k-1,\,a_k),\\
&\text{reward} && r(s_k,a_k)=\mathbf 1[k=1]\,R(a_k).
\end{aligned}\right.
\end{equation}
The only reward is received at the final step. A trajectory $\tau=(x_{t_K},\ldots,x_{t_0})$
therefore has probability $P_\theta(\tau)=\prod_{k=1}^K\pi_\theta(x_{t_{k-1}}\mid s_k)$, its
endpoint has law $\pth(x_0)=\sum_{\tau:\,\tau_0=x_0}P_\theta(\tau)$, and reward fine-tuning
maximizes the expected return $J_R(\theta)=\E_{\tau\sim P_\theta}[R(x_0)]$ of the $K$-step
sampler.

\begin{figure}[!b]
\refstepcounter{idrfalgorithm}\label{alg:idrf-training}
\begin{algobox}{\textbf{Algorithm \theidrfalgorithm\ \ IDRF training}\hfill\small reward fine-tuning of a $K$-step student}
\begin{algorithmic}[1]
\Require frozen reference $\fref$; reward $R$; weight $\lambda$; steps $K$; group size $M$
\State $G_\theta\gets\fref$,\ \ $\fpsi\gets\fref$ \Comment{zero gap at initialization}
\For{each training iteration}
  \algsection{Auxiliary: fit the student's current distribution}
  \For{each auxiliary update}
    \State $x_{t_j}\gets$ partial rollout of the MDP~\eqref{eq:mdp}, $j\sim\mathcal U\{1,\dots,K\}$ \Comment{no gradients}
    \State $\psi\gets\psi-\eta\,\nabla_\psi\Lmdlm\bigl(\fpsi,\operatorname{sg}[G_\theta(x_{t_j})]\bigr)$ \Comment{student frozen}
  \EndFor
  \algsection{Student: reward plus inverse distillation}
  \State $\theta_{\mathrm{old}}\gets\theta$;\ \ $\tau^{(1)},\ldots,\tau^{(M)}\sim P_{\theta_{\mathrm{old}}}$; record revealed tokens $U_k^{(i)}$ \Comment{full $K$-step rollouts}
  \State $A_i\gets R(x_0^{(i)})-\frac{1}{M-1}\sum_{i'\neq i}R(x_0^{(i')})$ \Comment{Eq.~\eqref{eq:reward-loss}}
  \For{each student update}
    \State $\widetilde x\gets G_\theta(x_{t_j})$ with $j\sim\mathcal U\{1,\dots,K\}$ on each $\tau^{(i)}$
    \State $\mathcal L\gets\LR(\theta)+\lambda\bigl[\Lmdlm(\fref,\widetilde x)-\Lmdlm(\fpsi,\widetilde x)\bigr]$ \Comment{Eqs.~\eqref{eq:reward-loss},~\eqref{eq:idrf}}
    \State $\theta\gets\theta-\eta\,\nabla_\theta\mathcal L$ \Comment{auxiliary frozen}
  \EndFor
\EndFor
\Ensure few-step student $G_\theta$; $\fpsi$ is discarded
\end{algorithmic}
\end{algobox}
\end{figure}

\paragraph{Reward term.}
Transitions are deterministic and $b_k$ does not
depend on $\theta$, so the policy gradient of $J_R$ involves only the tokens revealed at each step,
$U_k=\{\ell:\ x_{t_k}^\ell=\mask\neq x_{t_{k-1}}^\ell\}$, and never differentiates through discrete
draws or through $R$ (Appendix~\ref{app:policy-details}). Over $M$ rollouts generated by
$\theta_{\mathrm{old}}$, we optimize the clipped leave-one-out surrogate \citep{schulman2017ppo,kool2019buy,ahmadian2024back,agrpo2026}
\begin{equation}\label{eq:reward-loss}
\begin{aligned}
 &\LR(\theta)=-\frac1M\sum_{i=1}^M\sum_{k=1}^K\sum_{\ell\in U_k^{(i)}}
 \min\bigl\{\rho_{ik\ell}A_i,\ \operatorname{clip}(\rho_{ik\ell},1\pm\varepsilon)\,A_i\bigr\},\\
 &\rho_{ik\ell}=\frac{G_\theta^\ell(x_0^{(i),\ell}\mid x_{t_k}^{(i)})}
 {G_{\theta_{\mathrm{old}}}^\ell(x_0^{(i),\ell}\mid x_{t_k}^{(i)})},\qquad
 A_i=R(x_0^{(i)})-\frac{1}{M-1}\sum_{i'\ne i}R(x_0^{(i')}).
\end{aligned}
\end{equation}
For ancestral sampling, at $\theta=\theta_{\mathrm{old}}$, $-\nabla_\theta\LR$ is an unbiased estimate of $\nabla_\theta J_R$.
For max-confidence decoding, the recorded reveal sets are held fixed.

\paragraph{Inverse distillation on student states.}
The practical objective in Eq.~\eqref{eq:idrf} reuses the rollouts drawn for the reward. The
prediction $\widetilde x$ serves as a soft target on the simplex \citep{li2026idlm}, so each
cross-entropy in $\Lmdlm$ becomes $-\sum_{v\in\mathcal V}\widetilde x^\ell(v)\log f^\ell(v\mid z)$ for
$f\in\{\fref,\fpsi\}$ and a fresh corruption $z$ whose visible tokens are drawn from $\widetilde x$.
At revealed positions, $\widetilde x$ is the one-hot copied token. Sampling $j$ along the trajectory regularizes the student on the states its sampler visits, using a single auxiliary across depths. In our ablations, on-policy states work best: forward-masking completed
student samples lowers ATAC-Acc from $95.3\%$ to $52.5\%$ in DNA
design (Appendix~\ref{app:dna-ablations}). The regularizer costs one student evaluation plus one
forward pass through each of the reference and auxiliary models. We use a semi-gradient update for the inverse-distillation term: rollout states and sampled corruption contexts are held fixed, and gradients flow only through the student's soft target probabilities.
Algorithm~\ref{alg:idrf-training} summarizes training.

\section{Related Work}\label{sec:related}

\paragraph{Distillation of masked diffusion models.}
Several methods reduce the sampling steps of a pretrained masked diffusion model. SDTT distills
multiple teacher steps into fewer student steps \citep{deschenaux2025sdtt}, and Duo adapts
consistency distillation to uniform-state discrete diffusion \citep{sahoo2025duo}.
Distribution-matching methods instead train a few-step generator to match the teacher's
distribution, through an integral KL objective in DiDi-Instruct \citep{zheng2026didi} or inverse
distillation in IDLM \citep{li2026idlm}. These methods reproduce the teacher, whereas IDRF (ours) uses
inverse distillation to stay close to it while optimizing a reward.

\paragraph{Reward fine-tuning of masked diffusion models.}\label{sec:rw-reward}
To handle intractable sequence likelihoods, d1 and wd1 use denoising likelihood approximations
\citep{zhao2025d1,tang2025wd1}, while SPG uses lower and upper evidence bounds \citep{wang2025spg}.
Beyond reward fine-tuning, TUBE provides a further upper bound on the evidence of masked diffusion
models \citep{ivanov2026tube}.
Alternatively, AGRPO treats denoising as an MDP with tractable stepwise action probabilities
\citep{agrpo2026}, and TraceRL updates policies along recorded inference trajectories
\citep{wang2026tracerl}. DMPO and Discrete Tilt Matching learn reward-tilted distributions
\citep{zhu2026dmpo,chen2026dtm}. DRAKES instead uses differentiable relaxations for biological
sequence design when reward gradients are available \citep{wang2024drakes}.

\paragraph{Joint distillation and fine-tuning.}
In continuous diffusion, Diff-Instruct++ regularizes one-step reward optimization with an integral
KL \citep{luo2024diffinstructpp}, DMDR combines distribution matching with reinforcement learning
\citep{jiang2025dmdr}, and RTDMD uses reward-tilted distribution matching for few-step flow generators
\citep{huang2026rtdmd}. For masked image diffusion, Soft-Di[M]O refines a one-step generator through
soft embeddings. Its adversarial refinement uses real images, and its reward fine-tuning requires
differentiable rewards \citep{zhu2026softdimo}. IDRF (ours) uses inverse distillation
as the regularizer of trajectory-based policy gradients, supports reward fine-tuning directly in the few-step regime, and requires neither pretraining data nor reward gradients.

\section{Experiments}\label{sec:experiments}
We evaluate IDRF on regulatory-DNA design, CLIP-guided image generation, and sentiment steering of review text. Across these tasks, we study its ability to achieve high reward with few sampling steps and preserve sample quality through inverse-distillation regularization.

\subsection{Regulatory-DNA design}\label{sec:experiments-dna}
\paragraph{Setup.}
Following DRAKES \citep{wang2024drakes}, we generate $200$-bp regulatory sequences and
optimize predicted enhancer activity in HepG2 cells. IDRF students use $\lambda_{\mathrm{ID}}=0.2$
and sample in $K=8$ steps with ancestral or max-confidence decoding (Appendix~\ref{app:dna-train}). We report activity
under a held-out oracle (\textbf{Pred-Activity}), the percentage of sequences predicted
accessible in HepG2 (\textbf{ATAC-Acc}), \textbf{$3$-mer correlation} with highly active
natural enhancers, \textbf{Log-Lik} under the pretrained model, and per-sequence nucleotide
\textbf{Entropy} (Appendix~\ref{app:dna-metrics}). We re-evaluate CG \citep{nisonoff2024guidance}, SMC and TDS \citep{wu2023twisted}, CFG \citep{ho2022cfg}, and DRAKES with and without KL regularization using the released DRAKES implementation \citep{wang2024drakes}. We also evaluate the released SEPO checkpoints \citep{zekri2025sepo}. Each model is evaluated on $640$ sequences for each of three sampling seeds. Appendix~\ref{app:dna-baselines} compares our SEPO ATAC-Acc measurements with the published values. The rows for D$^2$-DPO \citep{borso2025d2dpo}, VRPO \citep{zhu2025llada15},
DDPP-IS \citep{rectorbrooks2025ddpp}, diffu-GRPO \citep{zhao2025d1} and SDPO, marked $^\ast$, are
reported by \citet{han2026sdpo}.

\begin{table}[!htbp]
\centering
\caption{\textbf{Regulatory-DNA design.} Unmarked values are our measurements, mean (sample SD)
over three sampling seeds. $\dagger$ marks the released SEPO checkpoint scored with the HepG2 ATAC output (Appendix~\ref{app:dna-baselines}), and $\ast$ results
reported by \citet{han2026sdpo}, whose released evaluation samples in $128$ steps; GF is gradient flow. Bold is best and underline second best among
our measurements, excluding the reward-only ablation; $\ast$ rows and Log-Lik are unranked.}
\label{tab:dna-results}
\footnotesize
\setlength{\tabcolsep}{1.5pt}
\renewcommand{\arraystretch}{1.05}
\begin{tabularx}{\linewidth}{@{}l*{5}{C}c@{}}
\toprule
Method & \shortstack{Pred-Activity\\$\uparrow$} & \shortstack{ATAC-Acc\\HepG2 (\%) $\uparrow$} & \shortstack{$3$-mer\\Corr. $\uparrow$} & \shortstack{Log-Lik\\\phantom{$\uparrow$}} & \shortstack{Entropy\\(nats)} & Steps$\downarrow$ \\
\midrule
Reference & $0.18\,(0.02)$ & $2.0\,(0.4)$ & $-0.042\,(0.010)$ & $-262\,(0.1)$ & $1.340\,(0.001)$ & $128$ \\
CG & $1.56\,(0.02)$ & $0.0\,(0.0)$ & $-0.057\,(0.002)$ & $-269\,(0.2)$ & $1.332\,({<}0.001)$ & $128$ \\
SMC & $4.69\,(0.48)$ & $49.7\,(6.4)$ & $\underline{0.880}\,(0.014)$ & $-256\,(1.0)$ & $1.299\,(0.012)$ & $128$ \\
TDS & $4.54\,(0.55)$ & $33.4\,(10.8)$ & $0.822\,(0.130)$ & $-257\,(1.7)$ & $1.313\,(0.024)$ & $128$ \\
CFG & $5.02\,(0.08)$ & $91.7\,(0.9)$ & $0.743\,(0.004)$ & $-265\,(0.7)$ & $1.259\,(0.003)$ & $128$ \\
\midrule
DRAKES & $5.61\,(0.07)$ & $\underline{92.3}\,(0.6)$ & $\mathbf{0.887}\,(0.001)$ & $-264\,(0.3)$ & $1.293\,(0.002)$ & $128$ \\
DRAKES w/o KL & $6.41\,(0.03)$ & $82.8\,(0.9)$ & $0.307\,(0.010)$ & $-281\,(0.1)$ & $1.358\,({<}0.001)$ & $128$ \\
SEPO + GF$^\dagger$ & $\mathbf{7.66}\,(0.01)$ & $47.4\,(1.8)$ & $0.640\,(0.001)$ & $-248\,(0.6)$ & $1.177\,(0.003)$ & $128$ \\
\midrule
D$^2$-DPO$^\ast$ & $2.97\,(0.03)$ & $35.6\,(0.9)$ & $0.944\,(0.002)$ & $-252\,(0.4)$ & -- & $128$ \\
VRPO$^\ast$ & $4.60\,(0.01)$ & $15.8\,(0.2)$ & $0.838\,(0.002)$ & $-255\,(0.8)$ & -- & $128$ \\
DDPP-IS$^\ast$ & $4.07\,(0.02)$ & $50.0\,(0.3)$ & $0.711\,(0.001)$ & $-253\,(0.9)$ & -- & $128$ \\
diffu-GRPO$^\ast$ & $5.86\,(0.04)$ & $33.0\,(0.8)$ & $0.783\,(0.001)$ & $-245\,(0.4)$ & -- & $128$ \\
SDPO$^\ast$ & $6.30\,({<}0.01)$ & $94.8\,({<}0.1)$ & $0.900\,(0.003)$ & $-246\,(0.5)$ & -- & $128$ \\
\midrule
\idrfrow\multicolumn{7}{@{}l}{\textbf{IDRF} (ours)} \\
\idrfrow\quad ancestral & $7.22\,(0.01)$ & $\mathbf{95.3}\,(0.7)$ & $0.683\,(0.001)$ & $-257\,(0.1)$ & $1.255\,(0.001)$ & $8$ \\
\idrfrow\quad max-conf & $\underline{7.31}\,(0.03)$ & $\underline{92.3}\,(0.5)$ & $0.718\,(0.003)$ & $-253\,(0.3)$ & $1.243\,(0.002)$ & $8$ \\
\textcolor{hackred}{RLOO w/o $\LID$} & $7.68\,({<}0.01)$ & $37.8\,(5.5)$ & $0.378\,(0.004)$ & $-204\,(0.6)$ & $1.036\,({<}0.001)$ & $8$ \\
\textcolor{hackred}{RLOO w/o $\LID$} & $7.43\,({<}0.01)$ & $0.1\,(0.1)$ & $0.240\,({<}0.001)$ & $-163\,(0.1)$ & $1.072\,({<}0.001)$ & $128$ \\
\bottomrule
\end{tabularx}
\end{table}

\begin{figure}[!b]
  \centering
  \includegraphics[width=\linewidth]{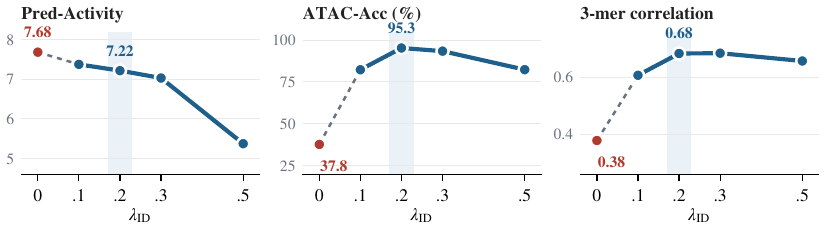}
  \caption{\textbf{Regularization sweep of IDRF (ours) on regulatory-DNA design} at $K=8$. Red marks
  reward-only fine-tuning ($\lambda_{\mathrm{ID}}=0$), and the shaded band the setting of
  Table~\ref{tab:dna-results}.}
  \label{fig:drakes-pareto}
\end{figure}

\paragraph{Accessible, high-activity sequences in eight steps.}
With eight steps, IDRF attains the top ATAC-Acc in Table~\ref{tab:dna-results} with ancestral
decoding ($95.3\%$, on par with $94.8\%$ for SDPO), and its max-confidence variant reaches the
second-highest activity ($7.31$). Every baseline samples in $128$ steps, $16\times$ more than IDRF.
The one method with higher activity, SEPO\,+\,GF ($7.66$), keeps only $47.4\%$ of its sequences
accessible, and SDPO reaches activity $6.30$.

\paragraph{Without inverse distillation, the oracle is exploited.}
Removing $\LID$ at the same eight steps raises activity to $7.68$ but keeps only $37.8\%$ ATAC-Acc
(against $95.3\%$), $3$-mer correlation $0.38$ (against $0.68$) and Entropy $1.04$, as the samples collapse into CA/CG-rich sequences almost entirely lacking T
(Appendix~\ref{app:dna-composition}). With the full
$128$-step sampler the collapse is complete ($0.1\%$ ATAC-Acc). The collapsed models receive higher estimated likelihoods ($-204$ and $-163$) than all other methods in our runs.

\paragraph{A controllable trade-off.}
Figure~\ref{fig:drakes-pareto} sweeps $\lambda_{\mathrm{ID}}$. Activity decreases
monotonically as the regularizer strengthens, while ATAC-Acc rises from $37.8\%$ to $95.3\%$ and
$3$-mer correlation from $0.38$ to $0.68$ by $\lambda_{\mathrm{ID}}=0.2$. A single weight selects
the operating point, and larger weights trade activity without further gains.

\subsection{CLIP-guided image generation}\label{sec:experiments-maskgit}
We fine-tune a class-conditional MaskGIT \citep{chang2022maskgit,besnier2025halton} pretrained on ImageNet with
CLIP image--text similarity \citep{radford2021clip} as the reward on ten ImageNet classes with
attribute prompts (Appendix~\ref{app:maskgit-prompts}), and hold out ImageReward
\citep{xu2023imagereward}, PickScore \citep{kirstain2023pickapic} and HPSv2.1 \citep{wu2023hpsv2}.
IDRF and its reward-only ablation both sample in $K=4$ steps, and each method is evaluated on
$1{,}000$ images. At four steps and without guidance, IDRF exceeds the $128$-step reference sampled with
classifier-free guidance ($w=3$) on every held-out metric, raising ImageReward from $-0.73$ to $0.88$
(Table~\ref{tab:maskgit-results}, Figure~\ref{fig:maskgit-compare}). Without $\LID$, the student reaches lower CLIP
($26.12$) and ImageReward ($-0.49$) and falls below the $128$-step reference on PickScore and HPSv2.1.
ReDGE \citep{gourevitch2026redge} reaches the highest CLIP by optimizing each image at inference,
yet stays below IDRF on all held-out metrics.

\begin{table}[!htbp]
\centering
\caption{\textbf{CLIP-guided MaskGIT generation}, $1{,}000$ images per method ($100$ per class),
mean with $95\%$ confidence interval over images. CLIP is optimized and the remaining metrics are
held out. $w$ is the classifier-free guidance weight, and rows without $w$, including IDRF, use no guidance ($w=0$). The Steps column reports the number of denoising steps. Four-step methods use an additional final unmasking step.
ReDGE adds $400$ reward-gradient updates per image.}
\label{tab:maskgit-results}
\small
\setlength{\tabcolsep}{3pt}
\resizebox{\linewidth}{!}{%
\begin{tabular}{lccccc}
\toprule
Method & Steps $\downarrow$ & CLIP $\uparrow$ & ImageReward $\uparrow$ & PickScore $\uparrow$ & HPSv2.1 ($\times100$) $\uparrow$ \\
\midrule
Reference & $4$ & $14.00\pm0.18$ & $-2.17\pm0.01$ & $18.24\pm0.03$ & $14.67\pm0.12$ \\
Reference, $w=3$ & $128$ & $23.11\pm0.13$ & $-0.73\pm0.05$ & $21.10\pm0.04$ & $25.05\pm0.12$ \\
ReDGE, $w=3$ & $20$ & $\mathbf{37.77}\pm0.10$ & $0.30\pm0.04$ & $21.38\pm0.03$ & $25.98\pm0.10$ \\
\textcolor{hackred}{RLOO w/o $\LID$} & $4$ & $26.12\pm0.10$ & $-0.49\pm0.03$ & $20.61\pm0.02$ & $22.57\pm0.08$ \\
\midrule
\idrfrow\textbf{IDRF} (ours) & $4$ & $27.74\pm0.06$ & $\mathbf{0.88}\pm0.02$ & $\mathbf{22.11}\pm0.02$ & $\mathbf{27.72}\pm0.04$ \\
\bottomrule
\end{tabular}%
}
\end{table}

\begin{figure}[!htbp]
  \centering
  \includegraphics[width=\linewidth]{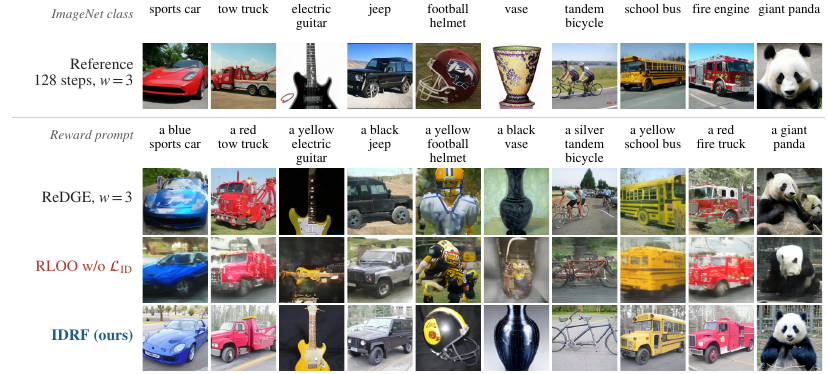}
  \caption{\textbf{CLIP-guided MaskGIT samples}, the first sample of each class for each method.
  Original ImageNet classes appear above the reference row, and reward prompts above the three
  guided rows. The reference is conditioned on the class alone. ReDGE optimizes the prompt reward
  at inference, while RLOO and IDRF (ours) are fine-tuned on it. IDRF uses four sampling steps plus cleanup.}
  \label{fig:maskgit-compare}
\end{figure}

\FloatBarrier
\subsection{Sentiment steering of review text}\label{sec:experiments-sentiment}
Text adds a failure mode that fixed-length outputs lack, because the model decides when to stop.
We fine-tune an MDLM \citep{sahoo2024mdlm}, trained on Amazon Fashion reviews
\citep{hou2024amazon} with an explicit EOS token, to maximize $\log P(1\,\text{star})$ under a
frozen BERT review
classifier (Appendix~\ref{app:amazon-setup}).
We hold out GPT-2-large perplexity \citep{radford2019gpt2} (\textbf{GenPPL}), the mean
\textbf{length} before the first EOS, the fraction of samples that emit \textbf{EOS}, the
token entropy of each sample divided by the log of its length (\textbf{Norm.\ entropy},
equal to $1$ when no token repeats), and \textbf{GM}, the gradient-moment distance of
\citet{hoogeboom2026dmmd} to $1{,}000$ held-out real one-star reviews (Appendix~\ref{app:amazon-gm}).

\begin{figure}[t]
\centering
\begin{minipage}[c]{0.33\linewidth}
  \centering
  \includegraphics[width=\linewidth]{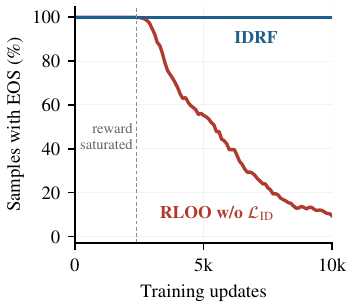}
\end{minipage}\hfill
\begin{minipage}[c]{0.645\linewidth}
\begin{tcolorbox}[samplebox, colframe=hackred, colbacktitle=hackred,
  title={RLOO w/o $\LID$\hfill no EOS, 127 of 127 tokens}]
Total junk this this trash is junk!!!!! lies junk not ever even 1 star buying Zero stars and garbage never buy this trash junk deserves 0 stars this junk no one ever buy this junk was junk!! 0 stars NEVER buy this junk 0 \textcolor{hackred}{[\dots]}
\samplesep{hackred}
Complete garbage garbage junk this deserve 0 stars this deserves a zero star I never ever I even allowed to give this zero star on this product. 0 stars 0 star this garbage buy this junk ever not ever ever buy this piece of \textcolor{hackred}{[\dots]}
\end{tcolorbox}\par\vspace{3pt}
\begin{tcolorbox}[samplebox, colframe=idrfblue, colbacktitle=idrfblue,
  title={IDRF (ours)\hfill 32 steps}]
This is absolutely junk!!! I purchased this brand and the first time I zipper it up the zipper broke! Please do not purchase this piece of junk! A waste of money!!!! I will never purchase from this company again!!!
\samplesep{idrfblue}
This is garbage! It is such horrible quality! I needed this floppy hat for a work event. It got folded in the washing machine and had dark scratches on the brim! Save your money! Don't waste more money on this!
\end{tcolorbox}
\end{minipage}
\caption{\textbf{Reward-only fine-tuning abandons the review format.} Left: fraction of samples emitting EOS during fine-tuning with $32$ sampling steps. Right: generated examples after $10{,}000$ updates, illustrating repetitive, unterminated RLOO outputs and complete IDRF reviews.}
\label{fig:amazon-eos}
\end{figure}

\paragraph{Reward-only fine-tuning abandons the review format.}
RLOO without the inverse-distillation regularizer achieves the highest reward but produces repetitive, unterminated text: $906$ of $1{,}000$ samples never emit EOS and reach the $127$-token limit (Figure~\ref{fig:amazon-eos}). Its normalized entropy drops to $0.72$, compared with $0.95$ for real reviews (Table~\ref{tab:amazon-results}). GenPPL alone does not capture this failure, assigning RLOO lower perplexity than the reference ($76.9$ versus $91.7$).

\paragraph{IDRF preserves the review format with $4$--$8\times$ fewer sampling steps.}
At both $16$ and $32$ steps, all evaluated IDRF samples emit EOS, with normalized entropy of $0.90$--$0.92$. At $32$ steps, IDRF attains a slightly lower reward than RLOO ($-0.044$ versus $-0.008$), but its GM distance to real one-star reviews is $2.6\times$ lower ($10.9$ versus $28.8$). It also achieves lower GenPPL than the $128$-step reference.

\begin{table}[!htbp]
\centering
\caption{\textbf{Sentiment steering on Amazon Fashion reviews.} $\log P(1\star)$ is optimized and the rest
held out. GM ($\pm$ s.e.) is against held-out real one-star reviews. Bold is the best fine-tuned model.}
\label{tab:amazon-results}
\footnotesize
\setlength{\tabcolsep}{1pt}
\renewcommand{\arraystretch}{1.0}
\begin{tabularx}{\linewidth}{@{}lc*{6}{C}@{}}
\toprule
Method & Steps$\downarrow$ & $\log P(1\star)\uparrow$ & GenPPL$\downarrow$ & Length & EOS\,(\%)$\uparrow$ & Norm.\ ent.$\uparrow$ & GM$\downarrow$ \\
\midrule
Real one-star & -- & $-0.950$ & $53.3$ & $32.2$ & $100.0$ & $0.948$ & $0.03\,(0.26)$ \\
MDLM reference & $128$ & $-5.070$ & $91.7$ & $27.8$ & $100.0$ & $0.955$ & $2.45\,(0.21)$ \\
\midrule
\textcolor{hackred}{RLOO w/o $\LID$} & $32$ & $\mathbf{-0.008}$ & $76.9$ & $126.8$ & $9.4$ & $0.719$ & $28.81\,(0.45)$ \\
\idrfrow\textbf{IDRF} (ours) & $32$ & $-0.044$ & $\mathbf{65.4}$ & $44.1$ & $\mathbf{100.0}$ & $0.902$ & $\mathbf{10.86}\,(0.50)$ \\
\idrfrow\textbf{IDRF} (ours) & $16$ & $-0.031$ & $107.3$ & $37.4$ & $\mathbf{100.0}$ & $\mathbf{0.922}$ & $11.44\,(0.46)$ \\
\bottomrule
\end{tabularx}
\end{table}

\FloatBarrier
\section{Discussion and Limitations}\label{sec:conclusion}
We introduced IDRF, a framework for reward fine-tuning of few-step masked discrete diffusion models using inverse-distillation regularization. Across DNA, image, and text generation, IDRF achieves high reward with fewer sampling steps while mitigating reward hacking. Future work includes scaling IDRF to larger models and extending it to additional domains.

\paragraph{Limitations.}
In practice, IDRF involves an adversarial min--max game between the student and the auxiliary denoiser. As in other adversarial training methods, this can introduce optimization instability and sensitivity to the balance between their updates. The auxiliary also adds training cost. Our practical surrogate admits a related KL guarantee for a mixture of one-step completion distributions, while controlling the final sampling distribution remains an open question.

\section*{AI use statement}
AI tools assisted with language editing, code implementation,  computational checks, and
mathematical analysis during preparation. The authors are responsible for the claims, results,
and final manuscript.

\section*{Ethics statement}
This work develops a method for reward fine-tuning of masked diffusion models and involves no human
subjects. All experiments use public datasets and pretrained models: the enhancer activity data of
\citet{gosai2024machine}, the ImageNet-trained MaskGIT \citep{chang2022maskgit}, and the Amazon
Reviews 2023 corpus \citep{hou2024amazon}. We do not redistribute review text. The DNA sequences
are scored only by predictive models and were not synthesized or tested experimentally. Reward
fine-tuning is dual-use. In particular, steering text toward a target sentiment could be misused to
produce deceptive reviews at scale, and our sentiment experiment serves only as a benchmark for
reward hacking.

\section*{Reproducibility statement}
Appendix~\ref{app:sequence-kl}
gives complete proofs of the theoretical results. Appendices~\ref{app:experimental-details},
\ref{app:maskgit-setup} and \ref{app:amazon-setup} list the models, data processing, hyperparameters,
sampling settings, evaluation metrics and seeds, and Appendix~\ref{app:dna-baselines} documents the
provenance of every baseline number.

\begingroup
\raggedright
\bibliographystyle{iclr2027_conference}
\bibliography{main}
\endgroup

\clearpage
\appendix
\section{Sequence-level inverse-distillation bound}\label{app:sequence-kl}
We prove Proposition~\ref{prop:sequence-kl} by expressing masked denoising as autoregressive
prediction along a random reveal order. The student enters only through its distribution
$\pth$ of completed sequences, in which its generation procedure, including any intermediate states
or trajectories, is marginalized. The masking used below is drawn independently after a clean
sample is obtained, so it describes the population denoising loss while leaving the student free
to use any sampler. As in Proposition~\ref{prop:sequence-kl}, the auxiliary is optimal.

\begin{proof}[{\normalfont\bfseries Proof of Proposition~\ref{prop:sequence-kl}}]
\leavevmode\par\nobreak
Each masked position in continuous-time SUBS is revealed at total rate $w_t$, because the
denoiser's categorical probabilities sum to one. We realize these events with independent
clocks $T_\ell$, $\Pr(T_\ell\le t)=1-\alpha_t$, whose reverse-time hazard is
$(-\alpha'_t)/(1-\alpha_t)=w_t$. The clocks are almost surely distinct. Let $\sigma$ be their
order of revelation, which is uniform over permutations by exchangeability \citep{shi2024md4}. Because the
denoiser takes no time input, the generated values depend on the order of the events and not
on their times.

For a clean sequence $x$, let $z_k(x,\sigma)$ retain the values at positions
$\sigma(1),\ldots,\sigma(k)$ and mask the rest, so $z_0=\mask^L$ and $z_L=x$. Revealing
positions in order $\sigma$ and drawing each from $f$ defines
\begin{equation}\label{eq:jump-sequence-law}
 \pi_{f,\sigma}(x)=\prod_{k=1}^L f^{\sigma(k)}\bigl(x^{\sigma(k)}\mid z_{k-1}(x,\sigma)\bigr),
 \qquad
 p_f(x)=\E_{\sigma}\,\pi_{f,\sigma}(x),
\end{equation}
with $\sigma$ uniform. The second expression is the continuous-time endpoint law, and in
particular $p_{\fref}=\pref$. Dependence between generated tokens is captured by the growing
visible context in each factor.

The NELBO in Eq.~\eqref{eq:mdlm} is the expected log-loss of these draws. Conditioned on
$T_\ell=t$, every other position is independently visible with probability $\alpha_t$ just before
position $\ell$ is revealed, which is exactly the forward-masked context at time $t$ with $\ell$
masked. The event density $-\alpha'_t=w_t(1-\alpha_t)$ equals the NELBO weight
times the probability that $\ell$ is masked. Summing over positions gives the any-order
autoregressive form of the NELBO \citep{ou2025radd},
\begin{equation}\label{eq:jump-nelbo}
 \Lmdlm(f,p)
 =\E_{x\sim p,\,\sigma}\Bigl[\sum_{k=1}^L-\log f^{\sigma(k)}\bigl(x^{\sigma(k)}\mid z_{k-1}(x,\sigma)\bigr)\Bigr]
 =\E_{x\sim p,\,\sigma}\bigl[-\log\pi_{f,\sigma}(x)\bigr],
\end{equation}
with $x$ and $\sigma$ independent. The schedule drops out, since time serves only as an
integration variable. The cross-entropy decomposition then yields
\begin{equation}\label{eq:jump-cross-entropy}
 \Lmdlm(f,\pth)=H(\pth)+\E_{\sigma}\KL{\pth}{\pi_{f,\sigma}}.
\end{equation}

Let $V(z)$ be the visible positions of a masked state $z$. All clean sequences agreeing with
$z$ produce this mask pattern with the same probability, which cancels in Bayes' rule, so the
optimal auxiliary prediction at a masked position $\ell$ is
\begin{equation}\label{eq:jump-bayes}
 f^{\mathrm B,\ell}_\theta(v\mid z)=\Pr_{x\sim\pth}\bigl(x^\ell=v\mid x^{V(z)}=z^{V(z)}\bigr).
\end{equation}
Along any reveal order its predictions are the chain-rule conditionals of $\pth$, so
$\pi_{f^{\mathrm B}_\theta,\sigma}=\pth$ for every $\sigma$. By the assumed expressiveness of the
auxiliary family, the minimum in Eq.~\eqref{eq:jump-cross-entropy} is therefore $H(\pth)$.
Subtracting it and applying convexity of KL in its second argument gives
\begin{equation}\label{eq:jump-kl-bound}
 \LID(\theta)=\E_{\sigma}\KL{\pth}{\pi_{\fref,\sigma}}
 \;\geq\;\KL{\pth}{\E_{\sigma}\pi_{\fref,\sigma}}=\KL{\pth}{\pref}.
\end{equation}
The finite reference NELBO ensures that the gap is finite and $\pth\ll\pref$. If the reference
is the Bayes denoiser of $\pref$, the same chain rule gives $\pi_{\fref,\sigma}=\pref$ for every
$\sigma$, and equality holds.
\end{proof}

The argument extends verbatim to time-conditioned denoisers by carrying the event times
$T_{\sigma(k)}$ alongside $\sigma$ through Eqs.~\eqref{eq:jump-sequence-law}--\eqref{eq:jump-kl-bound}.
Equation~\eqref{eq:jump-bayes} shows that the optimal auxiliary is time-agnostic in either case.

\subsection{The practical objective}\label{app:practical-bound}
Fix $\theta$, let $J\sim\mathcal U\{1,\dots,K\}$, and let $X\mid J=j$ follow the law
$\rho_{\theta,j}$ of the rollout state $x_{t_j}$ of the current student. For a state $x$ with
prediction $\widetilde x=G_\theta(x)$, one-hot at revealed positions, the one-step completion law and
its mixture over depths are
\begin{equation}\label{eq:completion-mixture}
Q_{\theta,x}(y)=\prod_{\ell=1}^L\widetilde x^\ell(y^\ell),\qquad
q_j=\E_{X\sim\rho_{\theta,j}}Q_{\theta,X},\qquad
\bar q_\theta=\frac1K\sum_{j=1}^Kq_j.
\end{equation}
The soft-target NELBO of Section~\ref{sec:reward-term} is
\begin{equation}\label{eq:soft-nelbo}
S(f,\widetilde x)=\E_{t,M,Y^V}\Bigl[w_t\sum_{\ell\in M}-\sum_{v\in\mathcal V}\widetilde x^\ell(v)\log f^\ell(v\mid Z)\Bigr],
\end{equation}
where $t\sim U(0,1)$, each position enters the masked set $M$ independently with probability
$1-\alpha_t$, the tokens $Y^\ell$ at the visible positions $\ell\in V$ are drawn independently from
$\widetilde x^\ell$, and $Z$ carries $Y^V$ on $V$ and $\mask$ on $M$.

\begin{propbox}
\begin{corollary}[Practical inverse-distillation bound]\label{cor:practical}
Assume the conditions of Proposition~\ref{prop:sequence-kl} with $\pth$ replaced by $\bar q_\theta$,
and let the auxiliary objective and the regularizer use the same law of $(J,X)$ and the same
corruption. Then, for every auxiliary with finite NELBO,
\begin{equation}\label{eq:practical-identity}
\E_{J,X}\bigl[S(\fref,G_\theta(X))-S(\fpsi,G_\theta(X))\bigr]
=\Lmdlm(\fref,\bar q_\theta)-\Lmdlm(\fpsi,\bar q_\theta).
\end{equation}
If the auxiliary minimizes its objective, this expectation equals the population loss of
$\bar q_\theta$ and satisfies
\begin{equation}\label{eq:practical-bound}
\Lmdlm(\fref,\bar q_\theta)-\min_\psi\Lmdlm(\fpsi,\bar q_\theta)\ \geq\ \KL{\bar q_\theta}{\pref}.
\end{equation}
For the sampler of Eq.~\eqref{eq:mdp}, $q_1=\pth$.
\end{corollary}
\end{propbox}

\begin{proof}[{\normalfont\bfseries Proof of Corollary~\ref{cor:practical}}]
\leavevmode\par\nobreak
Fix $J=j$ and $X=x$. Under $Q_{\theta,x}$ the positions are independent, and $Z$ depends on $Y$
only through $Y^V$. For $\ell\in M$, the token $Y^\ell$ is therefore independent of $(t,M,Y^V)$, and
$\E[-\log f^\ell(Y^\ell\mid Z)\mid t,M,Y^V]=-\sum_{v\in\mathcal V}\widetilde x^\ell(v)\log f^\ell(v\mid Z)$.
Hence $S(f,\widetilde x)=\Lmdlm(f,Q_{\theta,x})$. The NELBO is an expectation over the clean
sequence and is linear in its law, so averaging over $X$ and $J$ gives
$\E_{J,X}S(f,G_\theta(X))=\Lmdlm(f,\bar q_\theta)$. Applying this to $\fref$ and $\fpsi$ proves
Eq.~\eqref{eq:practical-identity}. By the same identity, the auxiliary objective equals
$\Lmdlm(\fpsi,\bar q_\theta)$, and its minimum over a family containing the Bayes denoiser of
$\bar q_\theta$ is attained there. Proposition~\ref{prop:sequence-kl} applied to the law
$\bar q_\theta$ gives Eq.~\eqref{eq:practical-bound}. Finally, $b_1=1$ in Eq.~\eqref{eq:mdp}, so the
last transition fills every remaining mask independently from $G_\theta^\ell(\cdot\mid x_{t_1})$ and
copies revealed tokens. The endpoint given $x_{t_1}$ therefore has law $Q_{\theta,x_{t_1}}$, and
averaging over $x_{t_1}$ gives $q_1=\pth$.
\end{proof}

The corollary characterizes the expected loss on rollout states drawn from the current student. Positions are independent given a rollout state, and the mixture $\bar q_\theta$ keeps the
dependence between tokens.

\section{Sampling and policy-gradient details}\label{app:policy-details}
The ancestral derivation uses a fixed schedule and a final transition to $t_0=0$.
If a numerical implementation stops at a positive noise floor, any remaining masks are
filled by a final categorical draw. Its sampled choices and log-probabilities must also
be included in the recorded trajectory and reward loss, and this extra evaluation counts
toward the sampling cost. For max-confidence sampling, masked positions are ranked
by their maximum predicted probability and revealed according to a fixed step budget.
Recomputing their probabilities with the recorded reveal sets held fixed gives the surrogate
described in Section~\ref{sec:reward-term}.

We use linear ancestral time grids ending at $10^{-5}$ for DNA and $10^{-3}$ for images and text. For MaskGIT, the reward loss uses four transitions sampled uniformly without replacement from the full recorded trajectory, including cleanup when present, with the sum scaled by the trajectory length divided by four. Inverse-distillation states exclude cleanup inputs for DNA and MaskGIT but include them for text. DNA max-confidence sampling reveals $25$ positions per step over eight steps, with categorical token sampling.

To verify the policy gradient, expand the path probability of the MDP in Eq.~\eqref{eq:mdp}.
Because $G_\theta$ assigns no mass to $\mask$, staying masked and being revealed are disjoint
events, and
$P_\theta(\tau)=c(\tau)\prod_{k=1}^K\prod_{\ell\in U_k}G_\theta^\ell(x_0^\ell\mid x_{t_k})$ with
$c(\tau)=\prod_{k=1}^K b_k^{|U_k|}(1-b_k)^{m_k-|U_k|}$, where $m_k$ is the number of masked
positions in $x_{t_k}$. The factor $c(\tau)$ is independent of $\theta$.
Differentiating $\sum_\tau P_\theta(\tau)R(x_0)$ gives, for any trajectory-independent baseline $b$,
\begin{equation}\label{eq:traj-logp}
\nabla_\theta J_R(\theta)
=\E_{\tau\sim P_\theta}\Bigl[(R(x_0)-b)\sum_{k=1}^K\sum_{\ell\in U_k}
 \nabla_\theta\log G_\theta^\ell(x_0^\ell\mid x_{t_k})\Bigr].
\end{equation}
The leave-one-out baseline is independent of trajectory $i$, conditional on the task,
so its expected product with that trajectory's score is zero. At
$\theta=\theta_{\mathrm{old}}$, every ratio $\rho_{ik\ell}$ in Eq.~\eqref{eq:reward-loss} equals one
and clipping is inactive, so differentiating Eq.~\eqref{eq:reward-loss} yields
the negative leave-one-out score-function estimate. After batch reuse, the clipped objective
acts as a surrogate for the reward gradient.

\section{Regulatory-DNA design}\label{app:experimental-details}

\subsection{Evaluation}\label{app:dna-metrics}
Every unmarked row of Table~\ref{tab:dna-results} evaluates one fixed checkpoint on $640$ sequences for each
of three sampling seeds and reports the mean and sample SD of the three per-seed values.
Pred-Activity is the median prediction of the held-out activity oracle. ATAC-Acc is the percentage
of sequences whose predicted HepG2 accessibility (output $1$ of the ATAC oracle) exceeds $0.5$.
$3$-mer correlation is the Pearson correlation between pooled $3$-mer counts of the samples and of
the $736$ enhancers of \citet{gosai2024machine} above the $99.9$th percentile of measured HepG2 activity. Log-Lik is the
median NELBO-based likelihood estimate of the pretrained model, with $128$ evaluation steps and one
Monte Carlo sample. Entropy averages the per-sequence nucleotide entropy
\begin{equation}\label{eq:dna-entropy}
 H_n=-\sum_{b\in\{\mathrm{A,C,G,T}\}}p_{n,b}\ln p_{n,b},\qquad 0\leq H_n\leq\ln4\approx1.386,
\end{equation}
where $p_{n,b}$ is the frequency of base $b$ in sequence $n$. The $16\times$ ratio in the main text
counts reverse steps, $128$ for every baseline and $8$ for IDRF.

\subsection{Training configuration}\label{app:dna-train}
All IDRF and reward-only runs start the student and the auxiliary from the pretrained DRAKES model, use its
log-linear noise schedule, and optimize the DRAKES fine-tuning oracle, while Pred-Activity uses the
held-out evaluation oracle. Each iteration takes $4$ auxiliary updates on on-policy states, then
draws one batch of rollouts and takes $2$ clipped student updates on it
(Table~\ref{tab:dna-hparams}). We report the final checkpoint. The reward-only run at $K=128$
accumulates gradients over $8$ batches of $32$ rollouts and forms the leave-one-out baseline within
each batch. The ancestral IDRF student is trained on one NVIDIA RTX 6000 GPU. For ancestral sampling, training uses categorical sampling for the final unmasking step. Evaluation follows the original DRAKES implementation and selects the most probable token at remaining masked positions.

\begin{table}[!htbp]
\centering
\caption{Training configuration for regulatory-DNA design.}
\label{tab:dna-hparams}
\vspace{3pt}
\small
\begin{tabular}{lccc}
\toprule
Setting & \multicolumn{2}{c}{IDRF (ours)} & RLOO w/o $\LID$ \\
\midrule
Sampler & ancestral & max-confidence & ancestral \\
Sampling steps $K$ & $8$ & $8$ & $8$ / $128$ \\
$\lambda_{\mathrm{ID}}$ & $0.2$ & $0.2$ & $0$ \\
Rollouts per iteration & $256$ & $256$ & $256$ / $256$ ($8\times32$) \\
Leave-one-out group & $256$ & $256$ & $256$ / $32$ \\
\midrule
Optimizer (student, auxiliary) & \multicolumn{3}{c}{AdamW, learning rate $10^{-4}$, weight decay $0$} \\
Gradient-norm clipping & \multicolumn{3}{c}{$1.0$} \\
Clipping $\varepsilon$ in Eq.~\eqref{eq:reward-loss} & \multicolumn{3}{c}{$0.2$} \\
Auxiliary / student updates per iteration & \multicolumn{3}{c}{$4$ / $2$} \\
Iterations & \multicolumn{3}{c}{$500$} \\
$\LID$ state & \multicolumn{3}{c}{$x_{t_j}$, $j\sim\mathcal U\{1,\dots,K\}$} \\
Training seed & \multicolumn{3}{c}{$0$} \\
\bottomrule
\end{tabular}
\end{table}

\subsection{Baseline provenance and evaluation discrepancies}\label{app:dna-baselines}
\paragraph{Baselines.}
Reference is the pretrained model without fine-tuning. Classifier guidance (CG, \citealp{nisonoff2024guidance})
and two sequential Monte Carlo methods \citep{wu2023twisted} steer it at inference, SMC with the
pretrained model as proposal and the twisted diffusion sampler (TDS) with CG as proposal. Classifier-free guidance
(CFG, \citealp{ho2022cfg}) samples a separately trained reward-conditioned model with guidance, as in DRAKES. D$^2$-DPO \citep{borso2025d2dpo} and
VRPO \citep{zhu2025llada15} adapt offline preference optimization to discrete diffusion. DRAKES
\citep{wang2024drakes} backpropagates the reward through the generated trajectory with
Gumbel-softmax relaxations. diffu-GRPO \citep{zhao2025d1} and SEPO \citep{zekri2025sepo} are policy
gradient methods, with SEPO\,+\,GF denoting its gradient flow (GF) variant. DDPP-IS \citep{rectorbrooks2025ddpp} matches the reward-tilted posterior with
importance sampling, and SDPO \citep{han2026sdpo} decomposes trajectory alignment into stepwise
objectives.

\paragraph{Provenance.}
We re-evaluated the DRAKES baselines using the authors' released code and configurations without modification.
Table~\ref{tab:dna-results} uses our common-protocol measurements for every unmarked
row. Table~\ref{tab:dna-reported-vs-measured} retains the published ATAC-Acc values for
transparency. The two columns have different provenance and are not treated as
measurements from the same run.

\begin{table}[!htbp]
\centering
\caption{\textbf{Published and independently measured ATAC-Acc (\%) $\uparrow$.}
Published values follow the cited papers. Our values use HepG2 output $1$ and
$640$ sequences per sampling seed. Parentheses give the reported SD for the papers
and sample SD across sampling seeds for our measurements.}
\label{tab:dna-reported-vs-measured}
\vspace{3pt}
\small
\setlength{\tabcolsep}{6pt}
\begin{tabular*}{\linewidth}{@{\extracolsep{\fill}}lccc@{}}
\toprule
Method & Published & Our measurement & Evaluated model \\
\midrule
SEPO & $99.5\,(0.2)$ & $8.70\,(1.49)$ & Released checkpoint \\
SEPO + GF & $99.9\,(0.09)$ & $47.40\,(1.82)$ & Released checkpoint \\
\bottomrule
\end{tabular*}
\end{table}

\begin{samepage}
\paragraph{ATAC output selection in SEPO.}
The SEPO paper defines ATAC-Acc for HepG2 \citep{zekri2025sepo}, but its released
\href{https://github.com/ozekri/SEPO/blob/4a4a40fef140002f6c819e09bc2c6a8db8d53e25/GRPO_MDLM_DNA/eval_plots.ipynb}{\texttt{eval\_plots.ipynb}}
selects output $0$ in the ATAC plot and threshold calculations.
The shared oracle's checkpoint metadata identifies output $0$ as GM12878 and output
$1$ as HepG2. The original
\href{https://github.com/ChenyuWang-Monica/DRAKES/blob/2d030ee33372aee8a81663aa0d9d8544c2458a27/drakes_dna/eval.ipynb}{DRAKES evaluator}
uses output $1$, as does the separate \texttt{eval.ipynb} in SEPO's repository.
We use HepG2 output $1$ for every method. The released SEPO and SEPO + GF checkpoints
then yield $8.70\%$ and $47.40\%$, respectively, under our sampling protocol.
Table~\ref{tab:dna-results} reports only SEPO + GF, which scores higher on Pred-Activity,
ATAC-Acc, $3$-mer correlation and Log-Lik alike. Plain SEPO scores $7.60\;(0.01)$ Pred-Activity, $8.7\;(1.5)$ ATAC-Acc,
$0.461\;(0.002)$ $3$-mer correlation, $-252\;(0.1)$ Log-Lik and $1.185\;(0.003)$ nats of
entropy at $128$ steps.
This identifies a mismatch in the public plotting/evaluation notebook. It does not
establish the execution history of every published result, nor explain differences
in Activity, $3$-mer correlation, or Log-Lik.
\par\end{samepage}

\subsection{On-policy states}\label{app:dna-ablations}
Table~\ref{tab:dna-state} replaces rollout states with forward-masked completed samples, masking the
student's sample $x_0$ at the same depth $t_j$; the soft target is the student's prediction at the
masked state. For ancestral decoding both give the same law of mask patterns and differ in how
visible values relate to the rest. On a rollout, the visible tokens were drawn first and the student
completes them. After forward masking, the visible tokens come from a finished sample. The resulting states differ from those encountered during generation. For max-confidence decoding the mask patterns also differ,
because the router reveals confident positions first. Under both samplers, ATAC-Acc and $3$-mer
correlation drop sharply.

\begin{table}[!htbp]
\centering
\caption{State on which $\LID$ is evaluated in IDRF (ours), regulatory-DNA design at $K=8$.}
\label{tab:dna-state}
\vspace{3pt}
\small
\begin{tabular}{llcccc}
\toprule
Sampler & $\LID$ state & Pred-Activity $\uparrow$ & ATAC-Acc (\%) $\uparrow$ & $3$-mer Corr.\ $\uparrow$ & Log-Lik \\
\midrule
Ancestral & on-policy $x_{t_j}$ & $7.22$ & $\mathbf{95.3}$ & $\mathbf{0.683}$ & $-257$ \\
Ancestral & forward-masked $x_0$ & $5.89$ & $52.5$ & $0.273$ & $-261$ \\
\midrule
Max-conf & on-policy $x_{t_j}$ & $7.31$ & $\mathbf{92.3}$ & $\mathbf{0.718}$ & $-253$ \\
Max-conf & forward-masked $x_0$ & $6.23$ & $42.8$ & $0.408$ & $-224$ \\
\bottomrule
\end{tabular}
\end{table}

\subsection{Compositional degeneration under reward-only fine-tuning}\label{app:dna-composition}
Table~\ref{tab:dna-composition} reports the nucleotide composition of the sequences generated by the
ancestral IDRF student and the two reward-only ablations of Table~\ref{tab:dna-results}. Each method
contributes $1{,}920$ sequences of $200$ nucleotides ($640$ per sampling seed), saved in FASTA
format, a plain-text format with one header line and one nucleotide string per sequence.
Dinucleotide frequencies count adjacent pairs within each sequence.

\begin{table}[!htbp]
\centering
\small
\setlength{\tabcolsep}{4pt}
\caption{\textbf{Reward-only fine-tuning nearly eliminates thymine.} Nucleotide and dinucleotide
composition in percent. The pair column is the combined frequency of CG, GC, CA and AC among all
dinucleotides, and T-free is the share of sequences without T.}
\label{tab:dna-composition}
\begin{tabular}{lrrrrrr}
\toprule
Method & A & C & G & T & \shortstack{CG+GC\\+CA+AC} & T-free \\
\midrule
\idrfrow IDRF, $K=8$ & 11.70 & 30.99 & 42.64 & 14.67 & 44.46 & 0.00 \\
RLOO w/o $\LID$, $K=8$ & 19.26 & 47.38 & 33.36 & ${<}0.01$ & 72.66 & 99.32 \\
RLOO w/o $\LID$, $K=128$ & 25.39 & 43.79 & 30.82 & ${<}0.01$ & 79.61 & 99.74 \\
\bottomrule
\end{tabular}
\end{table}

Reward-only fine-tuning leaves only $13$ T bases at $K=8$ and $5$ at $K=128$ across $1{,}920$
sequences, against $56{,}344$ for IDRF, so $99.3\%$ and $99.7\%$ of its sequences contain no T. The
CA/CG-related pairs rise from $44.5\%$ of all dinucleotides under IDRF to $72.7\%$ and $79.6\%$.
Figure~\ref{fig:dna-composition-samples} shows the first two sequences of sampling seed $0$ for each
method, without selection.

\begin{figure}[!htbp]
\centering
\begin{tcolorbox}[samplebox, colframe=idrfblue, colbacktitle=idrfblue, title={IDRF (ancestral)\hfill $K=8$}, fontupper=\fontsize{8}{10}\selectfont\ttfamily]
{\normalfont\scriptsize Sequence 1}\par
AGGAACTTCGCTGAGGCGGGCGGGGGGCGGGACGCGCGCGCCTCGCGAGGGCGGCGCAATTGGGGCGTGGGTGCGCGGCG\\CTGACGACGTCACTGCAGCGGCGCTGTGGCGCCGGCGGCGCGGGGCGCGCTTCGCGACGTCATGTCACGACGGCGTCGCG\\ACGTGCGTCTGACTGGGTTAGTCATCGCCCGCTGCGCGTG
\samplesep{idrfblue}
{\normalfont\scriptsize Sequence 2}\par
GAGGGCTCGGCGCATGGGCGTGTACTTTGCGGGGCGGAGCGTCGCGGCGTGCGCGCACGCGCGCGCGGGGCGCGTCACTT\\TCGTGTGCGGGCTGAGGTGGGTCACTGCCTAAGGCGGCGGGGGGGGGGGGCGGAGCTTACGTCGTTGTCGGTTACGTCAC\\GCGGCGCCGTCATCTGACGCGCGCCGGCGCCGTCCAGGTG
\end{tcolorbox}\par\vspace{4pt}
\begin{tcolorbox}[samplebox, colframe=hackred, colbacktitle=hackred, title={RLOO w/o $\LID$\hfill $K=8$}, fontupper=\fontsize{8}{10}\selectfont\ttfamily]
{\normalfont\scriptsize Sequence 1}\par
GGGCGGCCGGCGCCCCGGGCCGCCGGCGGAGACGCGGACGCGCAGACGCAGACGCGCACGCACGCACGCACGCACGCACG\\CACGCACGCCCGCACGCACCCACCCACGCACGCGCACGCGCGCACGCACCGACGCACGCACCGACGCCGCAGACGCAGAC\\ACACGCAGAGGCAGACGCAGAGGCGCAGCCGCCCCCCCCC
\samplesep{hackred}
{\normalfont\scriptsize Sequence 2}\par
GGGCGCGGGCGCGGCCGCCGCCGGGGCGGCGGCAGGCGCACGCGCAGACGCAGACGCAGACGCGCCCGCACGCACGCACG\\CACGCACGCACCCACGCAGGGCACGCACGCACCCACGCACGCACGCACGCACGCACGCACGCACACACGCACGCACGCAC\\GGACGCAGAGACGACGCAGACGCACACGCGCCCCCCCCCC
\end{tcolorbox}\par\vspace{4pt}
\begin{tcolorbox}[samplebox, colframe=hackred, colbacktitle=hackred, title={RLOO w/o $\LID$\hfill $K=128$}, fontupper=\fontsize{8}{10}\selectfont\ttfamily]
{\normalfont\scriptsize Sequence 1}\par
GGGGGAAGGGGGAAGGCCCCGGGCCGGGCCCGGGCCGGGGCCGCCGCGAGGCAGACGCACAGACGCAGACGCGCACGCGC\\ACACGCGCACGCACGCACGCACACACGCACACGCACACGCGCACGCACGCACGCACGCACGCACGCACGCGCACGCACGC\\ACGCACACACGCACACACACACACACACACACACACACAC
\samplesep{hackred}
{\normalfont\scriptsize Sequence 2}\par
GGGGGAGGGGGGAAAGCGCCGCCGCGGCGGGCCGGCGGCCGCGACGCCGACGCAGACGCAGAGACGCACACGCACGCACG\\CGCACGCACGCACGCACGCACGCACGCACGCACGCACGCACACGCACGCACGCACGCACGCACGCACGCACGCACGCACA\\CGCACGCACGCACGCACACACACACACACACACACACACA
\end{tcolorbox}\par\vspace{4pt}
\caption{\textbf{Unfiltered DNA outputs}, the first two sequences of sampling seed $0$ for each
method. Line breaks are for display only.}
\label{fig:dna-composition-samples}
\end{figure}

\section{CLIP-guided MaskGIT}\label{app:maskgit-setup}\label{app:maskgit-prompts}
\paragraph{Model and rewards.}
The reference is the class-conditional Halton-MaskGIT-L transformer \citep{besnier2025halton}, a MaskGIT
\citep{chang2022maskgit} pretrained on ImageNet at $384\times384$ resolution, which generates a $24\times24$ grid of VQ tokens
from a codebook of $16{,}384$ entries. The student and the auxiliary start from it, and the reference
and the VQ decoder stay frozen. The model receives only the ImageNet class, and the attribute prompt
of Table~\ref{tab:maskgit-prompts} enters only through the reward, $100\times$ the CLIPScore of a
CLIP ViT-L/14 \citep{radford2021clip}. ImageReward, PickScore v1 and HPSv2.1 are never optimized.

\paragraph{Training.}
Table~\ref{tab:maskgit-hparams} lists the configuration. Training uses one NVIDIA RTX 6000 GPU. Each batch holds $16$ groups of $4$
rollouts, each group for one class drawn uniformly. The student samples with $4$ ancestral steps on
a linear grid, followed by one cleanup call that fills any remaining masks. The reward loss uses
$4$ transitions drawn from each recorded trajectory, rescaled to the full trajectory, and each
rollout batch serves a single student update, so every ratio in Eq.~\eqref{eq:reward-loss} equals
one. The reward-only ablation uses the same configuration with $\lambda_{\mathrm{ID}}=0$.

\begin{table}[!htbp]
\centering
\caption{Training configuration for IDRF (ours) on CLIP-guided MaskGIT.}
\label{tab:maskgit-hparams}
\vspace{3pt}
\small
\begin{tabular}{lc}
\toprule
Setting & Value \\
\midrule
Sampling steps $K$ & $4$ ancestral steps and one cleanup call \\
$\lambda_{\mathrm{ID}}$ & $0.5$ \\
Optimizer (student, auxiliary) & AdamW, learning rate $3\times10^{-5}$, weight decay $0.01$ \\
Rollouts per update & $64$ ($16$ groups of $4$) \\
Auxiliary / student updates per iteration & $1$ / $1$ \\
Iterations & $80$ \\
Guidance weight during training & $w=0$ \\
$\LID$ state & $x_{t_j}$, $j\sim\mathcal U\{1,\dots,K\}$ \\
Precision & bfloat16 transformer, float32 logits \\
Training seed & $20$ \\
\bottomrule
\end{tabular}
\end{table}

\paragraph{Evaluation.}
Each method generates $100$ images for each of the ten classes. Scores are averaged within each
class and then across classes, and the $95\%$ confidence interval pools the within-class
variances with a Student-$t$ quantile. Classifier-free guidance combines logits as
$(1+w)\,\ell_{\mathrm{cond}}-w\,\ell_{\mathrm{uncond}}$. The $128$-step reference uses the Halton
sampler \citep{besnier2025halton} at temperature $1$ with $w=3$. ReDGE \citep{gourevitch2026redge} uses the ReMDM sampler \citep{wang2025remdm} with
$20$ network calls and $20$ logit updates per step ($400$ in total, learning rate $0.5$, KL weight
$1$), the same $100\times$ CLIPScore reward computed per image, and $w=3$.

\begin{table}[!htbp]
\centering
\caption{ImageNet classes and CLIP prompts used in the MaskGIT experiments.}
\label{tab:maskgit-prompts}
\vspace{3pt}
\small
\setlength{\tabcolsep}{5pt}
\begin{tabular}{rll}
\toprule
ImageNet ID & Class & CLIP prompt \\
\midrule
817 & Sports car & a blue sports car \\
864 & Tow truck & a red tow truck \\
546 & Electric guitar & a yellow electric guitar \\
609 & Jeep & a black jeep \\
560 & Football helmet & a yellow football helmet \\
883 & Vase & a black vase \\
444 & Tandem bicycle & a silver tandem bicycle \\
779 & School bus & a yellow school bus \\
555 & Fire engine & a red fire truck \\
388 & Giant panda & a giant panda \\
\bottomrule
\end{tabular}
\end{table}

\section{Sentiment steering of review text}\label{app:amazon-setup}
\paragraph{Data and SFT.}
We use the Amazon Fashion split of the Amazon Reviews 2023 corpus \citep{hou2024amazon}.
We remove duplicate reviews after trimming leading and trailing whitespace. Each training sequence contains a single review.
We retain nonempty reviews with at most $126$ GPT-2 tokens and reserve $1\%$ for validation using a deterministic hash-based split, without truncation or sentiment filtering. Each training sequence starts with $\textsc{bos}$, followed by the review text, and is padded with $\textsc{eos}$ tokens to length $128$. All $127$ non-BOS positions are noised and supervised. The reward is the log-probability of the one-star class under the frozen
classifier \url{https://huggingface.co/LiYuan/amazon-review-sentiment-analysis}. Starting from
the MDLM OpenWebText checkpoint
\citep{sahoo2024mdlm} with the GPT-2 tokenizer, we run EOS-aware SFT for $50{,}000$ steps at
learning rate $3\times10^{-4}$ and take the EMA weights. SFT uses Adam with global batch size $512$, $2{,}500$ warmup steps, EMA decay $0.99$, and seed $42$. This single checkpoint initializes
the frozen reference, the student and the auxiliary denoiser for every reward fine-tuning run
reported in Table~\ref{tab:amazon-results}.

\paragraph{Reward fine-tuning.}
Rollouts use ancestral sampling followed by a stochastic categorical terminal step, so
$K$ reverse updates cost $K+1$ denoiser evaluations. The score-function estimator of
Eq.~\eqref{eq:traj-logp} is applied to the full trajectory without subsampling reverse steps.
Each rollout batch receives a single on-policy update, so $\theta=\theta_{\mathrm{old}}$ and
the clipping in Eq.~\eqref{eq:reward-loss} is inactive here. Groups contain $8$ rollouts per
device, for a global batch of $64$ across $8$ NVIDIA A100 GPUs. Both models are trained with Adam at a learning rate of $10^{-6}$ with $100$ warmup steps. The auxiliary is updated once per student update. Evaluation runs every $100$ updates on $1{,}000$ fresh samples, and weights
are checkpointed every $500$ updates. The reward and GenPPL are read from the text up to the
first EOS, while the RLOO and inverse-distillation losses are computed on the full length-$128$
buffer, including the EOS tail. Every run starts from the same SFT checkpoint with training
seed $0$, and IDRF uses $\lambda_{\mathrm{ID}}=0.01$. IDRF16 is trained separately from IDRF32. All fine-tuned text models are evaluated after $10{,}000$ updates.

\paragraph{Evaluation.}
All generated rows of Table~\ref{tab:amazon-results}, including the reference, use the same
sampler, with ancestral updates followed by a stochastic categorical terminal step, evaluation
seed $1$ and $1{,}000$ samples. The real-review rows score $1{,}000$ held-out reviews that are disjoint
from the GM reference. GenPPL is the exponential of the token-averaged GPT-2-large NLL across all reviews, using BOS as context and including one EOS per review. \textbf{Norm.\ entropy} is the Shannon entropy of the GPT-2 token
frequencies within a sample divided by the log of its length, averaged over samples. BOS and EOS are excluded, and normalized entropy is set to zero for reviews containing fewer than two tokens. Repeated words account for a median of $66\%$ of word occurrences in RLOO reviews, $14\%$ in IDRF reviews, and $9\%$ in real one-star reviews. Word repetition is the fraction of word occurrences beyond their first occurrence, computed case-insensitively and summarized by the median across reviews.

\subsection{Gradient-moment diagnostic}\label{app:amazon-gm}
GM adapts the centered estimator of \citet[Eq.~14]{hoogeboom2026dmmd} to variable-length text.
For a candidate set and a reference set of $1{,}000$ reviews each, we compute the mean
gradient of GPT-2-large's token NLL on two disjoint subsets of $50$ reviews each from each set and
take the inner product of the two candidate-minus-reference differences. The reported value
averages $10$ such estimates, with one standard error across them. Each review contributes
its mean token NLL, conditioned on BOS and including one EOS. Gradients are computed in
float32 and inner products in float64. The estimator is unbiased for the squared distance
between mean gradients, so two samples of the same distribution score zero in expectation.

Table~\ref{tab:amazon-gm} uses two references from the held-out validation split. One holds
$1{,}000$ of the $2{,}748$ eligible real one-star reviews, and the other $1{,}000$ of the $21{,}448$
real reviews of every rating. Each reference is compared against a disjoint sample of the same
distribution, and both real-versus-real estimates ($0.03$ and $0.06$) are consistent with zero.
Against both references, IDRF at $32$ steps has a $2.4$--$2.7\times$ lower GM than reward-only fine-tuning.
These results suggest that inverse-distillation regularization reduces distributional drift under reward fine-tuning.

\begin{table}[!htbp]
\centering
\caption{Gradient moment against two held-out references (lower is closer), $\pm$ one
standard error over ten pair estimates. Generated sets are those of
Table~\ref{tab:amazon-results}.}
\label{tab:amazon-gm}
\vspace{3pt}
\small
\setlength{\tabcolsep}{6pt}
\begin{tabular}{lccc}
\toprule
Candidate set & Steps $\downarrow$ & GM vs.\ real one-star $\downarrow$ & GM vs.\ real, all ratings $\downarrow$ \\
\midrule
Real one-star reviews & -- & $0.03\pm0.26$ & $1.99\pm0.29$ \\
Real reviews, all ratings & -- & -- & $0.06\pm0.24$ \\
MDLM reference & $128$ & $2.45\pm0.21$ & $0.54\pm0.21$ \\
\midrule
RLOO w/o $\LID$ & $32$ & $28.81\pm0.45$ & $29.53\pm0.47$ \\
\idrfrow IDRF (ours) & $32$ & $10.86\pm0.50$ & $12.13\pm0.56$ \\
\idrfrow IDRF (ours) & $16$ & $11.44\pm0.46$ & $12.85\pm0.37$ \\
\bottomrule
\end{tabular}
\end{table}

\end{document}